\documentclass{article}

\usepackage[preprint]{neurips_2021}
\usepackage[utf8]{inputenc} 
\usepackage[T1]{fontenc}    
\usepackage{hyperref}       
\usepackage{url}            
\usepackage{booktabs}       
\usepackage{amsfonts}       
\usepackage{nicefrac}       
\usepackage{microtype}      
\usepackage{xcolor}         
\usepackage{amsmath}
\usepackage{amsthm}
\usepackage{graphicx}
\usepackage{algorithm}
\usepackage[noend]{algpseudocode}

\newtheorem{theorem}{Theorem}
\newtheorem{lemma}[theorem]{Lemma}
\newtheorem{definition}[theorem]{Definition}

\title{Gaussian Processes with Differential Privacy}

\author{%
  Antti Honkela and Laila Melkas\\
  Helsinki Institute for Information Technology HIIT\\
  Department of Computer Science\\
  University of Helsinki \\
  \texttt{\{antti.honkela,laila.melkas\}@helsinki.fi} \\
}

\newcommand{\bk}{\mathbf{k}}
\newcommand{\bm}{\mathbf{m}}
\newcommand{\bmu}{\boldsymbol{\mu}}
\newcommand{\bu}{\mathbf{u}}
\newcommand{\bx}{\mathbf{x}}
\newcommand{\by}{\mathbf{y}}
\newcommand{\bz}{\mathbf{z}}
\newcommand{\bd}{\mathbf{d}}
\newcommand{\bA}{\mathbf{A}}
\newcommand{\bB}{\mathbf{B}}
\newcommand{\bD}{\mathbf{D}}
\newcommand{\bE}{\mathbf{E}}
\newcommand{\bI}{\mathbf{I}}
\newcommand{\bK}{\mathbf{K}}
\newcommand{\bS}{\mathbf{S}}
\newcommand{\bV}{\mathbf{V}}
\newcommand{\bX}{\mathbf{X}}
\newcommand{\bY}{\mathbf{Y}}
\newcommand{\bZ}{\mathbf{Z}}
\newcommand{\bSigma}{\boldsymbol{\Sigma}}
\newcommand{\diff}{\;\mathrm{d}}

\begin{document}

\maketitle

\begin{abstract}
Gaussian processes (GPs) are non-parametric Bayesian models that are widely used for diverse prediction tasks. Previous work in adding strong privacy protection to GPs via differential privacy (DP) has been limited to protecting only the privacy of the prediction targets (model outputs) but not inputs. We break this limitation by introducing GPs with DP protection for both model inputs and outputs. We achieve this by using sparse GP methodology and publishing a private variational approximation on known inducing points. The approximation covariance is adjusted to approximately account for the added uncertainty from DP noise. The approximation can be used to compute arbitrary predictions using standard sparse GP techniques. We propose a method for hyperparameter learning using a private selection protocol applied to validation set log-likelihood. Our experiments demonstrate that given sufficient amount of data, the method can produce accurate models under strong privacy protection.
\end{abstract}

\section{Introduction}

Differential privacy (DP; \citealp{dwork_calibrating_2006}) is currently the standard approach for privacy-preserving machine learning. Based on the idea of bounding the impact of a single individual on the outcome, DP has been widely and successfully applied to models that treat each observation independently, ranging from deep learning \citep[e.g.][]{abadi_deep_2016} to Bayesian regression models \citep[e.g.][]{bernstein_differentially_2019} to generic Bayesian inference using Markov chain Monte Carlo \citep[e.g.][]{heikkila_differentially_2019} or variational inference \citep{jalko_differentially_2017}.

Gaussian processes (GPs; \citealp{rasmussen_gaussian_2006}) are a class of highly popular non-parametric Bayesian models, that do not fit this paradigm. GPs are often used to define a non-linear non-parametric regression model by modelling the (potentially multivariate) output value $\by = f(\bx)$ at each input $\bx$ using a distribution over \emph{functions} $f$. These are assumed to follow the GP $f \sim \mathcal{GP}(\mu, k)$ with some mean function $\mu$ (often assumed to be zero) and \emph{covariance} or \emph{kernel} $k$. Observations are usually assumed to include noise. By definition, the distribution of $f$ evaluated at any number of points follows a multivariate Gaussian distribution.

After fitting a GP to the data, the distribution of $f$ at each point depends directly on all observations. The dependence on inputs $\bx$ is especially complicated. This has caused the previous work in DP GPs \citep{smith_differentially_2018,smith_differentially_2019} to use DP mechanism for functional data \citep{hall_differential_2013} to address only the significantly weaker privacy model of label privacy, where only the privacy of the outputs ($\bY$) is considered while the inputs ($\bX$) are assumed non-sensitive. We seek a method to protect both $\bX$ and $\bY$. In order to do that, we formulate the GP learning problem in a way that has less complicated dependence on the input data $\bX$.

Our main novel contributions are:
\begin{itemize}
    \item A method for learning a variational sparse GP approximation under DP for both inputs and outputs.
    \item A method for approximately correcting the approximation posterior to account for DP noise.
    \item A method for learning DP-GP hyperparameters under DP.
\end{itemize}

\section{Background}

\subsection{Differential Privacy (DP)}

DP defines a statistical criterion for a randomised algorithm that bounds the accuracy of inferences an adversary can make about the inputs of the algorithm using its outputs, providing a guarantee of privacy to input data subjects. The guarantee applies regardless of additional side information the adversary may have, and degrades gracefully under repeated use of the data.

DP definition is based on the concept of neighbouring data sets $\bD \sim \bD'$. We use the \emph{substitute} neighbourhood, where $\bD \sim \bD'$ if $\bD'$ can be obtained from $\bD$ by substituting a single entry.

\begin{definition}
Given $\epsilon \ge 0, \delta \in [0,1]$, a randomised algorithm $\mathcal{M}$ is $(\epsilon, \delta)$-DP if for all neighbouring data sets $\bD \sim \bD'$ and for all measurable sets of outputs $O \subseteq \mathrm{Range}(\mathcal{M})$, we have
\begin{equation}
    \mathrm{Pr}[\mathcal{M}(\bD) \in O] \le e^{\epsilon} \mathrm{Pr}[\mathcal{M}(\bD') \in O] + \delta.
\end{equation}
\end{definition}
Lower values of $\epsilon$ and $\delta$ correspond to stronger privacy.

Our main tool for achieving DP is the analytical Gaussian mechanism \citep{balle_improving_2018}, which can release the value of a function $f(\bD)$ under $(\epsilon, \delta)$-DP by adding Gaussian noise whose variance is tuned to the $L_2$-sensitivity of $f$, $\Delta_f = \sup_{\bD \sim \bD'} \|f(\bD) - f(\bD')\|_2$.

\subsection{Sparse and Variational GPs}

Gaussian processes (GPs) provide a means for performing exact Bayesian inference over functions $f \sim \mathcal{GP}(\mu, k(\cdot, \cdot))$ \citep{rasmussen_gaussian_2006}. A GP defines a distribution over functions, which is specified by the mean function $\mu$, often assumed to be zero \emph{a priori}, and covariance or kernel function $k(\cdot, \cdot)$. Covariance function is the key parameter that defines the smoothness and other properties of the functions $f$.

The first GP learning problem is \emph{inference of the posterior GP} given noisy observations of some underlying function. Assuming the observations follow a Gaussian likelihood, it is possible to perform exact inference over $f$, but given $n$ observations this will require operations with $n \times n$ kernel matrix that require significant time and results in a complicated dependence on input data that is unfavourable to DP.
A popular method for reducing the computational complexity is to employ a \emph{sparse GP method} that seeks to replace the dependence on the full input $\bX$ with dependence on a smaller number of inducing inputs $\bZ$ and values of the function at those points $\bu = f(\bz)$ \citep{snelson_sparse_2005,candela_unifying_2005,titsias_variational_2009,bauer_understanding_2016}.

The \emph{variational sparse GP} first introduced by \citet{titsias_variational_2009} solves the problem by defining a variational distribution $q(\bu)$ over the function values at the inducing inputs. Assuming that the inducing inputs and kernel parameters are fixed and a Gaussian observation model with noise variance $\sigma^2$, the mean $\bm$ and covariance $\bS$ of the optimal approximating distribution for function values $\bu$ at $\bZ$, $q(\bu) = \mathcal{N}(\bu;\; \bm, \bS)$, can be written in closed form as \citep[][Eq. 10]{titsias_variational_2009}:
\begin{equation}
 \begin{split}
    \bm &= \sigma^{-2} \bK_{ZZ}\bSigma \bK_{ZX} \bY, \\
    \bS &= \bK_{ZZ} \bSigma \bK_{ZZ}, \\
    \bSigma &= (\bK_{ZZ} + \sigma^{-2} \bK_{ZX}\bK_{XZ})^{-1},
    \label{eq:q_u}
 \end{split}
\end{equation}
where $\bK_{ZZ} = (k(\bz_i, \bz_j))$ is a matrix of kernel values between inducing inputs and $\bK_{ZX} = \bK_{XZ}^T = (k(\bz_i, \bx_j))$ is a matrix of kernel values between inducing inputs and input data locations $\bX$. The variational approximation can be used to evaluate arbitrary predictions \citep[][Eq. 21]{hensman_scalable_2015}.

The above discussion assumes the kernel and all other hyperparameters of the model are known, which is usually not the case. The resulting second GP learning problem, \emph{hyperparameter learning}, is addressed in the variational framework by finding hyperparameters that maximise the variational evidence lower bound (ELBO) \citep{titsias_variational_2009}.

\section{GP Inference under Differential Privacy}

One key insight in making variational GPs DP is that we only need to publish $\bZ$ and $q(\bu)$ to allow using the model to make predictions on arbitrary new points. Assuming that $\bZ$ is chosen independently of the input data or using some other DP mechanism, it will be enough to ensure the privacy of $\bm$ and $\bS$ that define $q(\bu)$.

Following \eqref{eq:q_u}, $q(\bu)$ only depends on the private data $\bX, \bY$ via two terms:
\begin{equation}\label{eq:AB_def}
\begin{split}
    \bA &= \bK_{ZX} \bY = \sum_{i=1}^n \bk_i y_i,\\ 
    \bB &= \bK_{ZX}\bK_{XZ} = \sum_{i=1}^n \bk_i \bk_i^T, 
\end{split}
\end{equation}
where $\bk_i = \bK_{ZX_i} = (k(\bz_j, \bx_i))_{j=1}^{|Z|}$ is a vector of kernel values between $\bx_i$ and all inducing points $\bZ$.

Both of these terms are sums over data points $(\bx_i, y_i)$, suggesting one can implement efficient DP mechanisms for these by evaluating the sensitivity of the terms.

\subsection{Sensitivity Calculations}

Sensitivity calculations will depend on the kernel. Assuming a bounded kernel with $|k| \le \sigma_f^2$, we have a trivial bound
\begin{equation}\label{eq:sensbasic}
    \| \bk \|_2 \le \sqrt{|Z|} \sigma_f^2 =: R_k.
\end{equation}
Assuming scalar $y$ with $|y| \le R_y$, we can adapt the sensitivity analysis of \citet{kulkarni_differentially_2021} for the mechanism
\begin{equation}\label{eq:gaussmech}
    \mathcal{M}(\bX, \bY) = \begin{bmatrix} \bA \\ \hat{\bB} \end{bmatrix}
    + \mathcal{N}\left(0, \mathrm{diag}(\sigma_a^2 \bI_d, \sigma_b^2 \bI_{d_2})\right)
\end{equation}
to obtain the sensitivity
\begin{equation}\label{eq:dpsens}
    \Delta = \sqrt{\frac{\sigma_b^2}{2 \sigma_a^2}R_y^4 + 2 R_y^2 R_k^2 + 2\frac{\sigma_a^2}{\sigma_b^2} R_k^4}.
\end{equation}
$\hat{\bB}$ denotes the vector $[b_{11} \cdots b_{dd} \quad \sqrt{2}b_{12} \cdots \sqrt{2} b_{d-1,d}]$ formed from upper diagonal elements of $\bB$ with off-diagonal elements scaled up, which can be used to reconstruct a full noisy symmetric $\bB$, $d = |Z|$ and $d_2 = \binom{d+2}{d}$. Full derivation of Eq. \eqref{eq:dpsens} is in the Supplementary material.

\subsection{Kernel-Dependent Sensitivity}

The bound $\| \bk_i \|_2 \le \sqrt{|Z|} \sigma_f^2$ is pessimistic. It is possible to obtain tighter bounds by making assumptions about the kernel and inducing input placement.

Let us assume that we have a stationary kernel that is decreasing as a function of the distance of the two inputs. This class includes most standard stationary kernels. We can write such kernels as $k(\bx, \bx') = \sigma_f^2 k_r(r)$, where $r = \| \bx - \bx'\|$ and $0 \le k_r(r) \le 1$.

Let us further assume that the minimum distance between two input points is at least $d_z$. In 1D we have for example
\begin{equation}\label{eq:sens1d}
    \frac{\| \bk_i \|_2^2}{\sigma_f^4}
    = \sum_{j=1}^{|Z|} k_r(\| \bz_j - \bx_i \|)^2
    \le \sum_{j'=0}^{z_{hf}} 2 k_r(j' d_z)^2 + (z_{hc} - z_{hf}) k_r(z_{hc} d_z)^2,
\end{equation}
where $z_{hf} = \lfloor|Z|/2\rfloor$ and $z_{hc} = \lceil|Z|/2\rceil$.

In higher dimensions an accurate analysis becomes considerably more difficult. We can still obtain a simple bound on the kernel values by noting that we can have at most one inducing point at a distance smaller than $d_z/2$. This yields the bound
\begin{equation}\label{eq:sensnd}
    \frac{\| \bk_i \|_2^2}{\sigma_f^4}
    = \sum_{j=1}^{|Z|} k_r(\|\bz_j - \bx_i\|)^2
    \le 1 + (|Z|-1) k_r(d_z/2)^2.
\end{equation}

A tight upper bound can be computed with specific information on the kernel and the locations of the inducing points. If we assume, for example, that the inducing points form a regular grid and that the kernel can be written as a product of $N$ 1-dimensional exponentiated quadratic kernels and variance $\sigma_f^2$. Squared kernel norm can then be expressed as
\begin{equation}\label{eq:sensqe}
\begin{split}
 \| \bk_i \|^2 &= \sigma_f^4 \sum_j \exp \left( - \sum_{n}\frac{\| x_{in} - z_{jn} \|^2}{2\ell_n^2}\right)^2 \\
 &= \sigma_f^4 \sum_j \exp \left( - \sum_n \left\| \frac{x_{in}}{\ell_n} - \frac{z_{jn}}{\ell_n} \right\|^2 \right) \\
 &= \sigma_f^4 \sum_j \exp \left( - \| \tilde{\bx} - \tilde{\bz}_j \|^2 \right).
\end{split}
\end{equation}

Scaling the inducing points by scalars does not affect their relative locations. Taking the gradient and the Hessian of this form, we get
\begin{align}
 &\nabla_{\tilde{\bx}} \| \bk \|^2 = - 2\sigma_f^4 \sum_j \exp \left( - \| \bd_j \|^2 \right) (\bd_j) \label{eq:sensgrad}\\
 &\frac{\partial^2 \| \bk \|^2}{\partial \tilde{\bx} \partial \tilde{\bx}^T} = 2\sigma_f^4 \sum_j \exp \left( - \| \bd_j \|^2 \right) \left( 2 \bd_j \bd_j^T - \bI \right), \label{eq:senshes}
\end{align}
where $\bd_j = \tilde{\bx} - \tilde{\bz}_j$.

Due to symmetry of the regular grid formed by the inducing points, it can be seen that the gradient is zero when $\bx_i$ is at the centre of the grid. If the number of grid points in each direction is odd, the centre point is a maximum. This can be checked by verifying that the Hessian \eqref{eq:senshes} is negative definite, which can be done with the available information on the kernel and the inducing points. In general, additional analyses are still required to show formally that the centre point yields the global maximum.

\subsection{Noise-Aware DP Posterior Approximation}

Applying the Gaussian mechanism to privatise \eqref{eq:AB_def}, we will replace $\bA$ and $\bB$ with $\bA + \bE_a$ and $\bB + \bE_b$, respectively, where $\bE_a$ and $\bE_b$ are the noise terms. A naive application of these to the update rules in Eq. \eqref{eq:q_u} would replace $\Sigma$ with $(\bK_{ZZ} + \sigma^{-2}(\bK_{ZX}\bK_{XZ} + \bE_b))^{-1}$, which can become non-positive definite if the noise $\bE_b$ is large. This would make the results basically meaningless, as $\bS$ would not be a valid covariance matrix anymore.

In order to avoid this issue, the update needs to be moderated by adding a diagonal regularisation term to obtain the DP update for $q(\bu) = \mathcal{N}(\bu;\; \bm, \bS)$:
\begin{equation}
 \begin{split}
    \bm &= \sigma^{-2} \bK_{ZZ}\tilde{\bSigma} (\bA + \bE_a),\\
    \bS &= \bK_{ZZ} \tilde{\bSigma} \bK_{ZZ} + \bS_2,\\
    \tilde{\bSigma} &= (\bK_{ZZ} + \sigma^{-2} (\bB + \bE_b) + \lambda \bI)^{-1},
    \label{eq:q_u_DP}
 \end{split}
\end{equation}
where $\lambda$ is a constant selected to make sure $\bK_{ZZ} + \sigma^{-2} (\bB + \bE_b) + \lambda \bI$ remains positive definite and $\bS_2$ is added to the covariance to account for the impact of DP noise to $\bm$.

In order to set $\lambda$, we use Lemma 6 of \citet{wang_revisiting_2018}, which shows that with probability $1-\rho$, $\|\bE_b\|^2 \le \sigma_b^2 d \log(2 d^2/\rho)$, where the dimensionality $d=|Z|$ is equal to the number of inducing points $\bZ$ and $\sigma_b^2$ is the variance of noise added to $\bB$.
Using this and taking into account the smaller variance of off-diagonal terms, we can set
\begin{equation}\label{eq:lambda}
    \lambda = \sigma_b \sigma^{-2} \sqrt{|Z| \log(2|Z|^2/\rho)} \frac{|Z|+1}{2|Z|}
\end{equation}
with suitable $\rho$ to have a high probability of $\tilde{\bSigma}$ being positive definite. We use $\rho = 0.01$ in the experiments. Unlike the AdaSSP mechanism of \citet{wang_revisiting_2018} this approach is inspired by, we do not use a DP estimate of the minimum eigenvalue, because the smallest eigenvalues of kernel matrices are typically very close to zero, which means that adding a DP estimate of the eigenvalue would mostly just add noise and drain the privacy budget.

To evaluate $\bS_2$, we observe that the DP mechanism adds two noise terms $\bE_a$ and $\bE_b$ that impact $\bm$. As the two noise terms are independent, their impacts can be evaluated separately. As $\bE_a \sim \mathcal{N}(0, \sigma_a^2 \bI)$ is purely additive, the covariance added by it can be computed easily to obtain
\begin{equation}
    \bS_{2,1} = \mathrm{Cov}[\sigma^{-2} \bK_{ZZ}\tilde{\bSigma} \bE_a] = \sigma_a^2 \sigma^{-4} \bK_{ZZ} \tilde{\bSigma}^2 \bK_{ZZ}.
\end{equation}
As $\bE_b$ is not additive, we estimate its impact by linearising the expression for $\bm$ in $(\bE_b)_{ij}$ to obtain
\begin{align}
    \bm &\approx \bm(\bE_b = \bE_b^*) + \frac{\partial \bm}{\partial (\bE_b)_{ij}} (\bE_b)_{ij} \\
    \frac{\partial \bm}{\partial (\bE_b)_{ij}} &= - \sigma^{-4} \bK_{ZZ} \tilde{\bSigma} \bE_{ij} \tilde{\bSigma} (\bA + \bE_a),
\end{align}
where $(\bE_{ij})_{kl} = \delta(i=k) \delta(j=l)$, i.e. it is a matrix of all zeros except a single 1 at position $(i,j)$.

This allows computing the additional covariance component due to $\bE_b$ as a sum of contributions of the independent elements $(\bE_b)_{ij}$ added to the upper triangular matrix as
\begin{multline}
    \bS_{2,2} = \sum_{i=1}^{|Z|} \sigma^{-8} \sigma_b^2 \bK_{ZZ} \tilde{\bSigma} \bE_{ii} \tilde{\bA} \bE_{ii} \tilde{\bSigma} \bK_{ZZ} \\
    + \sum_{i=1}^{|Z|} \sum_{j=i+1}^{|Z|} \sigma^{-8} \frac{\sigma_b^2}{2} \bigg[ \bK_{ZZ} \tilde{\bSigma} \bE_{ij} \tilde{\bA} \bE_{ji} \tilde{\bSigma} \bK_{ZZ} \\
    + \bK_{ZZ} \tilde{\bSigma} \bE_{ji} \tilde{\bA} \bE_{ij} \tilde{\bSigma} \bK_{ZZ} \bigg],
\end{multline}
where $\tilde{\bA} = \tilde{\bSigma} (\bA + \bE_a) (\bA + \bE_a)^T \tilde{\bSigma}$.

The final covariance of DP $q(\bu)$ is thus
\begin{equation}
    \bS = \bK_{ZZ} \tilde{\bSigma} \bK_{ZZ} + \bS_2 = 
    \bK_{ZZ} \tilde{\bSigma} \bK_{ZZ} + \bS_{2,1} + \bS_{2,2}.
\end{equation}
The additional terms in the covariance only depend on the previously released DP quantities and therefore have no additional impact on the privacy.

\subsection{Final Algorithm and its Privacy}

\begin{algorithm}
\begin{algorithmic}
\Function{DP-GP-inference}{$\bX, \bY, \bZ, k(\cdot, \cdot), \sigma$, $c$, $R_y, \epsilon, \delta$}
  \State Compute $R_k$ using Eq. \eqref{eq:sensbasic}, \eqref{eq:sens1d}, \eqref{eq:sensnd} or \eqref{eq:sensqe}
  \State Compute sensitivity $\Delta$ for DP-mechanism using Eq. \eqref{eq:dpsens} with $\sigma_a/\sigma_b=c$
  \State Compute $\sigma_a$ for $(\epsilon, \delta)$-DP $\Delta$-sensitive analytical Gaussian mechanism \citep{balle_improving_2018}
  \State $\sigma_b \gets \sigma_a / c$
  \State Compute $q(\bu) = \mathcal{N}(\bu;\;\bm, \bS)$ on $\bZ$ using Eq. \eqref{eq:q_u_DP}
  \State \Return $\bm, \bS$
\EndFunction
\end{algorithmic}
\caption{Algorithm for DP-GP inference using known hyperparameters.}
\label{alg:inference}
\end{algorithm}

\begin{theorem}
The DP-GP inference in Algorithm \ref{alg:inference} is $(\epsilon, \delta)$-DP with respect to the input data set $(\bX, \bY)$.
\end{theorem}

\begin{proof}
The only part of the algorithm using the input data is the computation of $q(\bu)$ using Eq.~\eqref{eq:q_u_DP}. This update only accesses the data via Eq.~\eqref{eq:AB_def}. These are made private using the Gaussian mechanism of Eq.~\eqref{eq:gaussmech} whose noise variance is calibrated using the given privacy parameters and the sensitivity computed above.
\end{proof}

\section{Hyperparameter Learning}

GPs depend on a number of hyperparameters such as observation noise level and kernel parameters. These are commonly either inferred by applying Markov chain Monte Carlo (MCMC) to obtain a posterior distribution or optimised. In both cases, the process commonly makes use of the marginal likelihood of the parameters, obtained by analytically marginalising the GP and available in closed form for GP regression, or in the case of variational approximation the variational evidence lower bound (ELBO).

While the marginal likelihood and the ELBO can be written in closed form, their expressions have a complex dependence especially on the input data $\bX$, making it difficult to evaluate these quantities under DP.

To avoid this issue, we will instead use validation set likelihood as the objective for hyperparameter optimisation. We further assume independence of the validation set points in the likelihood to make the expression easier to evaluate under DP.

\subsection{DP Evaluation of Validation Set Likelihood}

Using standard variational GP techniques \citep{hensman_scalable_2015}, given an approximation $q(\bu) = \mathcal{N}(\bu;\; \bm, \bS)$ over values at inducing points $\bZ$ and hyperparameters $\theta$, the predictive distribution at a new set of noise-free validation points $\bV$ can be computed as
\begin{equation}
 \begin{split}
    p(\by_V | \bV, \bX, \bY, \theta) 
    &\approx \int_{\bu} p(\by_V | \bV, \bu, \theta) p(\bu | \bX, \by, \theta) \diff \bu\\ 
    &\approx \int_{\bu} p(\by_V | \bV, \bu, \theta) q(\bu | \theta) \diff \bu,
 \end{split}
\end{equation}
which can be evaluated analytically to obtain normal $\mathcal{N}(\hat{\bmu}, \hat{\bSigma})$ with
\begin{equation}
 \begin{split}
    &\hat{\bmu} = \bK_{VZ} \bK_{ZZ}^{-1} \bm, \\
    &\hat{\bSigma} = \bK_{VV} - \bK_{VZ} \bK_{ZZ}^{-1} (\bK_{ZZ} - \bS) \bK_{ZZ}^{-1} \bK_{ZV}.
 \end{split}
\end{equation}
Under our assumption of independence, we will only consider the diagonal of $\hat{\Sigma}$. This yields the likelihood of noisy validation set $(\bV, \by_V)$ as
\begin{equation}\label{eq:predlog}
 \begin{split}
    \log p(\by_V | \bV, \bX, \bY, \theta)
    &\approx \log \mathcal{N}(\by_V;\; \hat{\bmu}, \mathrm{diag}(\hat{\bSigma}) + \sigma^2 \bI_{|V|}) \\ 
    &= \sum_{i=1}^{|V|} \log \mathcal{N}(y_{V,i};\; \hat{\mu}_i, \hat{\Sigma}_{ii} + \sigma^2),
 \end{split}
\end{equation}
which only depends on $(\bX, \by)$ via $(\bm, \bS)$ and where each term of the sum only depends on the corresponding element of the validation set.

The usual way to estimate the sum under DP would be to bound each term and add noise to the result using for example Laplace or Gaussian mechanism. This is not easy in our case because the normal log-likelihood cannot easily be bounded tightly over a wide range of parameter values, but at the same time the values of the terms in the sum are often quite similar as they depend strongly on the model hyperparameters $\theta$.

To avoid these complications, we apply a generic mean estimation method \textsc{CoinPress} \citep{biswas_coinpress_2020} that uses a multi-step approach with iteratively tightening bounds to allow accurate estimation even if the original bounds are very loose. Our estimate of the log-likelihood is simply $|V|$ times the mean estimate returned by \textsc{CoinPress}.

We use \textsc{CoinPress} with centre and radius of the interval of possible values defined as $R_{\text{CP}} = \frac{R_y^2}{\sigma^2}, C_{\text{CP}} = -\frac{1}{2} \log(2\pi) - \log \sigma - \frac{R_y^2}{\sigma^2}$. Values outside this interval are clipped to the bounds of the interval. The justification of the bounds is in Supplementary material.
\textsc{CoinPress} operates under the zero-concentrated DP model $\rho$-zCDP. Given $(\epsilon, \delta)$, we use \text{CoinPress} with
\begin{equation}\label{eq:zcdp_conversion}
    \rho_{\text{CP}} = \left(\sqrt{\epsilon + \log(1/\delta)} - \sqrt{\log(1/\delta)}\right)^2,
\end{equation}
which guarantees that the result will be $(\epsilon, \delta)$-DP \citep{bun_concentrated_2016}. 
We use \textsc{CoinPress} with 12 iterations with $3/4$ of $\rho$ budget assigned to the last iteration and the rest split evenly between the preceding iterations.
Under very small privacy budget, \textsc{CoinPress} sometimes returns values outside the interval. Such values are ignored in our use in Algorithm \ref{alg:hypers}.

\subsection{DP Learning of Optimal Hyperparameters}

Using these DP estimates of the validation set log-likelihood, it is in principle possible to apply MCMC or standard optimisation techniques to infer the hyperparameters. This approach is not very practical, as these methods typically require evaluating the log-likelihood for a large number of candidate values, and without additional measures each of these would carry a significant privacy cost.

To avoid this erosion of privacy guarantees, we apply the method for private selection from private candidates proposed by \citet{liu_private_2019}, which allows privately selecting the optimal hyperparameters from a large set of possible values with $\mathcal{O}(\epsilon)$ privacy when each candidate model evaluation is $\epsilon$-DP. The corresponding guarantees for $(\epsilon, \delta)$-DP candidate models are weaker, but still offer substantial improvement over naive composition.

In order to apply the private selection algorithm, we need to formulate the composition of GP inference and validation set likelihood evaluation as a single DP mechanism over the union of disjoint training and validation sets. This generalises Lemma 5.1 of \citet{liu_private_2019} to cases where the hyperparameter goodness measure is not a counting query but a general $(\epsilon, \delta)$-DP mechanism.

Following \citet{liu_private_2019}, we define $\mathcal{M}_i(D_1)$, $i=1,\dots,K$ as the $(\epsilon, \delta)$-DP mechanisms for GP inference using training set $D_1$ using different hyperparameter values. Mechanisms $q_i(m, D_2)$ are $(\epsilon, \delta)$-DP mechanisms to evaluate the validation set $D_2$ log-likelihood of the model $m \sim \mathcal{M}_i(D_1)$. The final mechanism is $Q(D_1, D_2) \sim (m, q_i(m, D_2))$, where $i \sim \mathrm{Uniform}([K]), m \sim \mathcal{M}_i(D_1)$.

\begin{lemma}\label{lem:hyperpar_privacy}
Assume $D_1 \cap D_2 = \emptyset$, the mechanisms $\mathcal{M}_i(D_1), i=1,\dots,K$ are $(\epsilon, \delta)$-DP with respect to $D_1$, and the mechanisms $q_i(m, D_2), i=1,\dots,K$ are $(\epsilon, \delta)$-DP with respect to $D_2$. Then the joint mechanism $Q(D_1, D_2)$ defined above is $(\epsilon, \delta)$-DP with respect to $D = (D_1, D_2)$.
\end{lemma}
\begin{proof}
Let us consider two neighbouring data sets $D = (D_1, D_2)$ and $D' = (D_1', D_2')$ that differ by a single element. Depending on whether the different element is in the $D_1$ or $D_2$ part, we have two cases.

If the differing element is in the $D_1$ part, we have $D_2 = D_2'$ and thus $q_i(m, D_2) = q_i(m, D_2')$. The output of $Q(D_1, D_2)$ is post-processing of the $(\epsilon, \delta)$-DP mechanism $\mathcal{M}_i(D_1)$ and thus satisfies $(\epsilon, \delta)$-DP.

If the differing element is in the $D_2$ part, we have $D_1 = D_1'$ and thus $\mathcal{M}_i(D_1) = \mathcal{M}_i(D_1')$. As the distribution of $m$ is thus same under both scenarios, we have for any measurable set $S \subseteq \mathrm{Range}(Q)$ and for all $m$ and $i$,
\begin{equation*}
    \mathrm{Pr}((m, q_i(m, D_2)) \in S)
    \le e^{\epsilon} \mathrm{Pr}((m, q_i(m, D'_2)) \in S) + \delta
\end{equation*}
by $(\epsilon, \delta)$-DP of $q_i(m, D_2)$. Therefore $Q(D_1, D_2)$ is $(\epsilon, \delta)$-DP with respect to $(D_1, D_2)$.
\end{proof}

We can now define Algorithm \ref{alg:hypers} to find optimal hyperparameters using a bounded step variant of Algorithm 2 of \citet{liu_private_2019}. The algorithm repeatedly draws randomly from $Q(D_1, D_2)$ and keeps a record of the highest scoring candidate. It outputs the highest scoring candidate after every draw with probability $\gamma > 0$, or after at most $T = \frac{1}{\gamma} \log \frac{1}{\delta_2}$ draws, where $\delta_2 \in [0,1]$. Lemma \ref{lem:hyperpar_privacy} as well as Theorem 3.4 of \citet{liu_private_2019} show that when both GP inference and validation set log-likelihood evaluation are $(\epsilon, \delta)$-DP, the output of the hyperparameter inference algorithm is $(\epsilon_{\text{tot}}, \delta_{\text{tot}})$-DP with
\begin{equation}\label{eq:hyper_eps_tot}
    \epsilon_{\text{tot}} = 3\epsilon + 3 \sqrt{2\delta}, \quad
    \delta_{\text{tot}} = \sqrt{2 \delta} T + \delta_2.
\end{equation}

\subsection{Final Algorithm and its Privacy}

\begin{algorithm}
\begin{algorithmic}
\Function{DP-GP-learn-hyperparameters}{$\bX_1$, $\bY_1$, $\bX_2$, $\bY_2$, $\bZ$, $R_y$, $\boldsymbol{\Theta}$, $\epsilon_{\text{tot}}$, $\delta_{\text{tot}}$, $\gamma$}
    \State $t_0 \gets \textsc{find-root}(t[1-\log(t)] - \delta_{\text{tot}}, t)$
    \State $\delta \gets \gamma^2t_0^2/2;\;\; \delta_2 \gets \sqrt{2\delta}/\gamma;\;\; \epsilon \gets \epsilon_{\text{tot}}/3 - \sqrt{2\delta}$
    \State $T \gets \frac{1}{\gamma} \log \frac{1}{\delta_2};\;\; v_{\text{opt}} \gets -\infty$
    \For {$t = 1,\dots,T$}
        \State $i \sim \mathrm{Uniform}(1, \dots, k)$
        \State $\bm, \bS \gets \textsc{DP-GP-inference}$($\bX_1$, $\bY_1$, $\bZ$, $k_i(\cdot, \cdot), \sigma_i, R_y, \epsilon, \delta)$
        \State $v_j \gets \log p(\bY_{2,j} | \bX_{2,j}, \bX, \bY, \theta_i)$, $\quad j=1,\dots,|\bX_2|$ using Eq. \eqref{eq:predlog}
        \State $\tilde{v} \gets |\bX_2| \textsc{CoinPress-mean}$($\textsc{clip}(v_j$, $C_{\text{CP}}$, $R_{\text{CP}})$, $C_{\text{CP}}$, $R_{\text{CP}}$, $\rho_{\text{CP}})$
        \If {$\tilde{v} \le C_{\text{CP}}+R_{\text{CP}}$ \textbf{ and } $\tilde{v} > v_{\text{opt}}$}
            \State $v_{\text{opt}} \gets \tilde{v};\;\; m_{\text{opt}} \gets (\bm, \bS)$
        \EndIf
        \If {$\texttt{rand}(1) < \gamma$}
            \State \Return ($m_{\text{opt}}, v_{\text{opt}}$)
        \EndIf
    \EndFor
    \State \Return ($m_{\text{opt}}, v_{\text{opt}}$)
\EndFunction
\end{algorithmic}
\caption{Algorithm for finding optimal DP-GP hyperparameters within a set of candidate kernels and noise variances $\boldsymbol{\Theta} = (\boldsymbol{\theta_1} = (k_1, \sigma_1), \dots, \boldsymbol{\theta}_k = (k_k, \sigma_k))$.}
\label{alg:hypers}
\end{algorithm}

\begin{theorem}
Algorithm \ref{alg:hypers} is $(\epsilon_{\text{tot}}, \delta_{\text{tot}})$-DP with respect to the combination of the training and validation sets $((\bX_1, \bY_1), (\bX_2, \bY_2))$.
\end{theorem}

The proof, which follows from Lemma \ref{lem:hyperpar_privacy} and Theorem 3.4 of \citet{liu_private_2019}, is included in the Supplementary material.

\begin{figure}[th]
    \centering
    \includegraphics[width=0.9\textwidth]{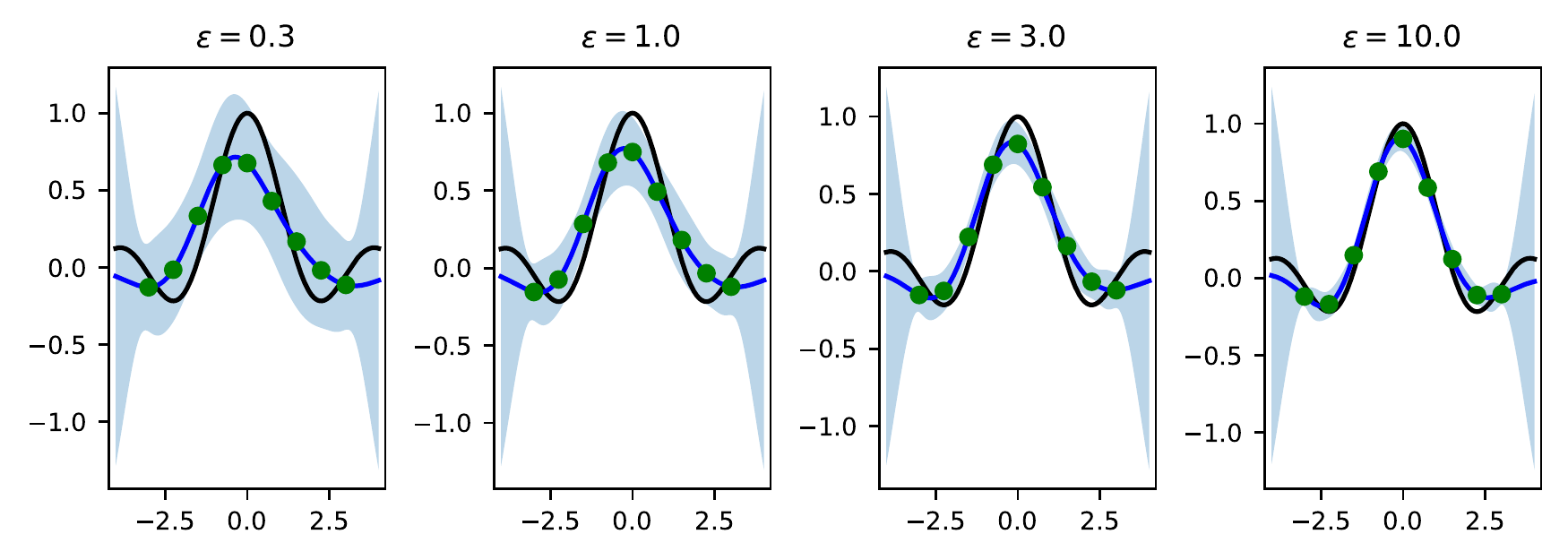}
    \caption{Illustration of DP-GP inference with varying $\epsilon$ and $\delta = 10^{-4}$ on a data set of $N=1024$ points sampled from the black function with noise. The blue line is the posterior mean for the GP with shaded region denoting 2 sd confidence region. Green dots denote means of inducing point values under $q(\bu)$.}
    \label{fig:illustration}
\end{figure}

\section{Experiments}

\subsection{Inference of a Noisy Function}
\label{sec:demo_experiment}

We first illustrate the DP-GP inference with a task of inferring function $f(x) = \sin(2x)/2x$ from $N=1024$ noisy observations on interval [-4, 4]. We place $|Z|=9$ inducing points on a regular grid between $-3$ and $3$. We assume the noise standard deviation $\sigma = 0.1$ is known. We use exponentiated quadratic (a.k.a.\ squared exponential) kernel with length scale $l = 1.0$. The results in Fig.~\ref{fig:illustration} show how the model becomes more accurate as $\epsilon$ increases and the privacy becomes less tight, but the posterior confidence intervals for the true function scale with the privacy level.

\begin{figure}[t]
    \centering
    \includegraphics[width=0.9\textwidth]{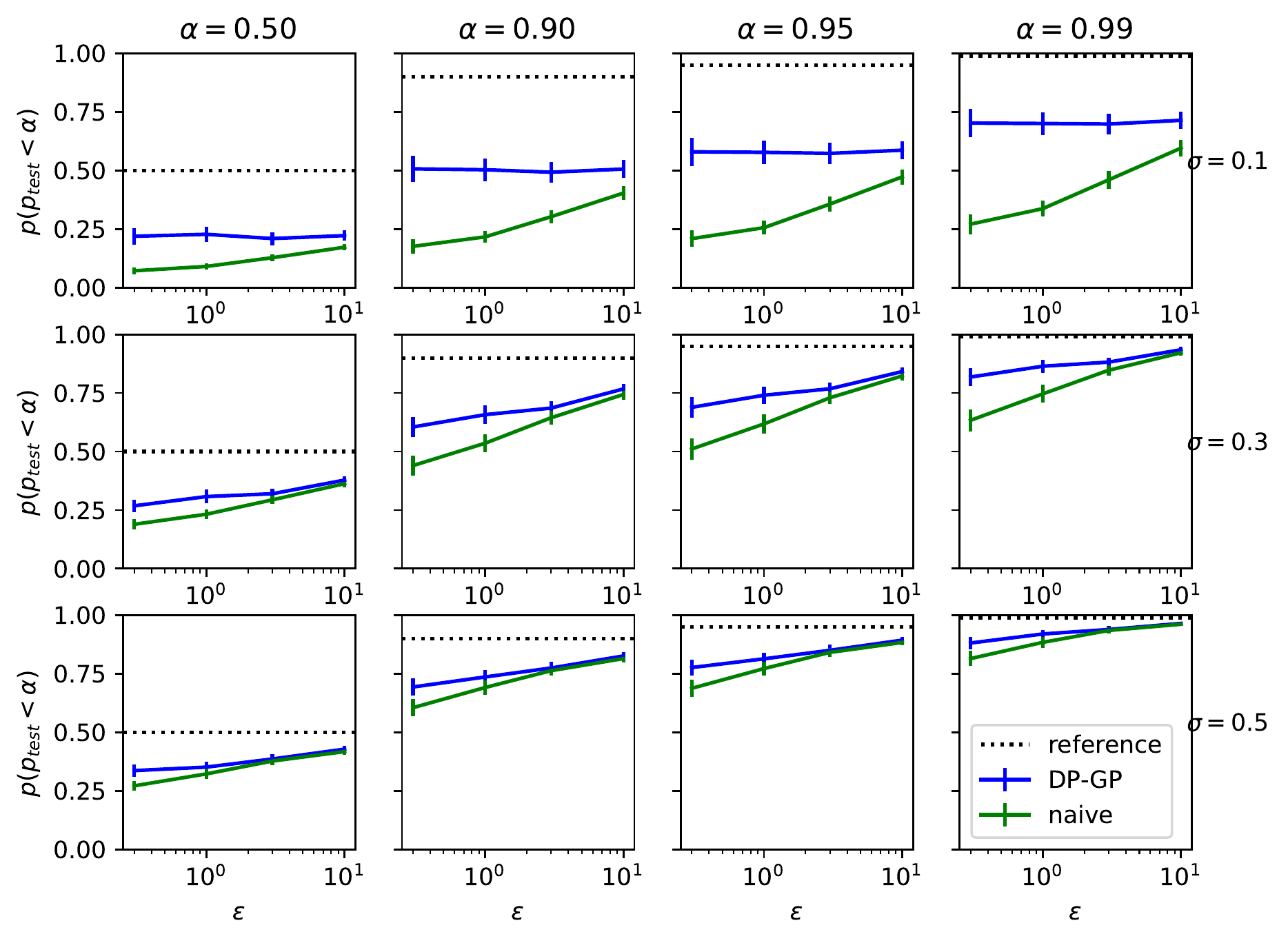}
    \caption{Uncertainty calibration accuracy of the proposed DP-GP algorithm and a naive algorithm omitting the additional posterior covariance contributions due to noise added for DP. The curves show the fraction of test points within $\alpha$ high posterior density region. The error bars indicate two standard errors of the mean after 40 repeats.}
    \label{fig:calibration}
\end{figure}

\subsection{Uncertainty Calibration}

In this experiment we test the calibration accuracy of the predictive uncertainty of the learned GP model.
To avoid model mismatch, we draw the target function from a GP with exponentiated quadratic covariance with $l=1.0$, and simulate $N_{\text{tot}}=1024$ noisy observations from this function on the interval $[-4, 4]$. The observations are divided to training and test sets of equal size, and a DP-GP is fitted using the training set to $|Z|=15$ inducing points regularly spaced on the interval $[-3.5, 3.5]$. We then evaluate predictions of the model for the test set points, and check the fraction of test set points that are within the central high posterior density region covering fraction $\alpha$ of the posterior. If the posterior uncertainty is properly calibrated, the observed fraction should be close to $\alpha$.

Fig.~\ref{fig:calibration} shows the calibration results for a range of values of interval width $\alpha$ (columns), observation noise standard deviations $\sigma$ (rows) and privacy levels $\epsilon$. The value $\delta = 10^{-4}$ is used in all cases. The results show a comparison of the proposed full model ("DP-GP") with a naive model that omits the DP uncertainty calibration in the posterior covariance ("naive") relative to the reference level $\alpha$ ("reference").

The results indicate that the additional terms significantly improve the calibration, especially for low observation noise and tight privacy where the calibration of the naive model is the poorest. The proposed model is still overconfident in many cases, possibly due to ignoring the error caused by the additional necessary regularisation from Eq. \eqref{eq:lambda}. It is noteworthy that the DP-GP calibration is significantly less sensitive to $\epsilon$, which suggests it is able to capture the additional uncertainty caused by DP. The runtime of the experiment using 40 repeats in 48 different scenarios was less than 30 seconds on a modern laptop CPU.

\begin{figure}[t]
    \centering
    \includegraphics[width=\textwidth]{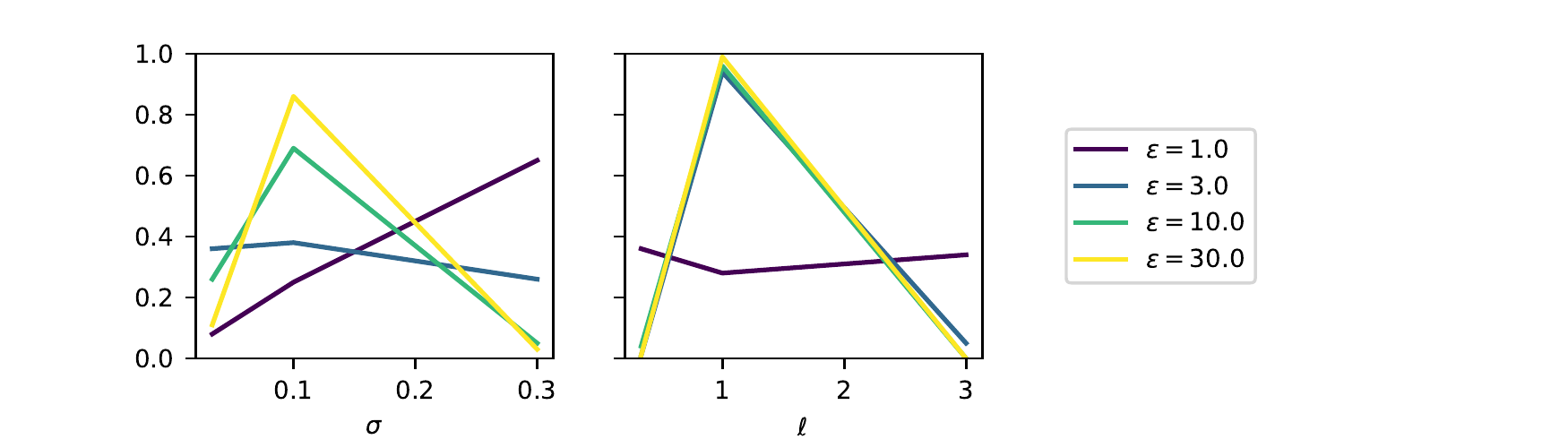}
    \caption{Probability of finding specific observation noise $\sigma$ and length scale $\ell$ values in hyperparameter learning experiment after 100 repeats.}
    \label{fig:hyperparams}
\end{figure}

\subsection{Hyperparameter Learning}

We demonstrate the accuracy of hyperparameter learning using a set of $N=2048$ points generated similarly as in Sec.~\ref{sec:demo_experiment}, split evenly to training and validation sets. We use a search among 9 combinations of 3 observation noise levels $\sigma \in \{0.1/3, 0.1, 0.3\}$ and 3 kernel length scales $\ell \in \{1/3, 1.0, 3.0\}$, where consecutive values differ by a factor of 3. The results shown in Fig.~\ref{fig:hyperparams} show that $\epsilon=3.0$ is needed for any sensible results and observation noise level estimates improve up to $\epsilon=30.0$. The accuracy and runtime can be traded off with parameter $\gamma$. 100 repeats with 4 values of $\epsilon$ took approximately 10 minutes on a modern laptop CPU with $\gamma = 0.01$.

\begin{figure}[t]
    \centering
    \includegraphics[width=0.8\linewidth]{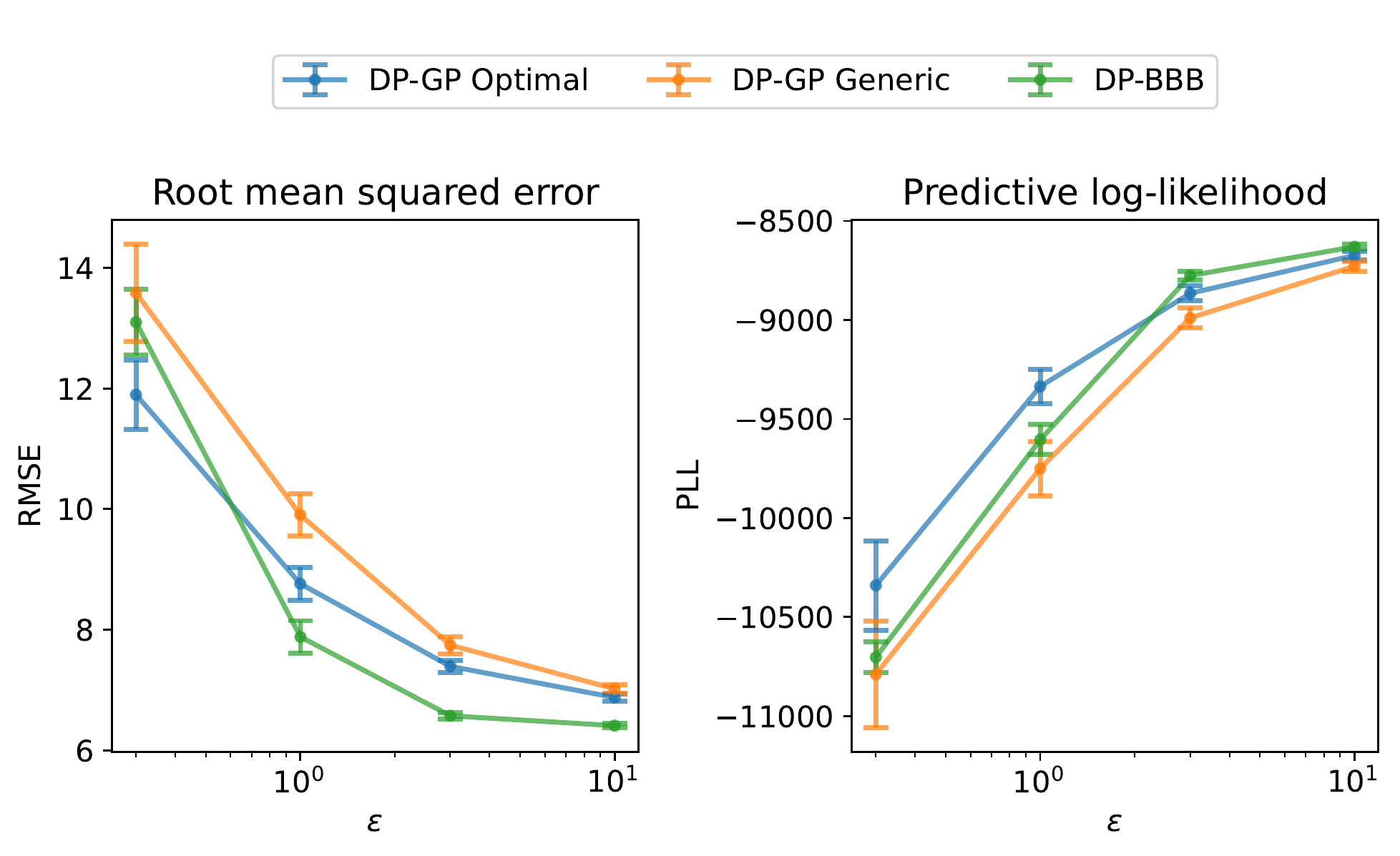}
    \caption{RMSE and predictive log-likelihood as functions of $\epsilon$ for DP-GP using two methods of computing kernel norm upper bound and DP-BBB. For DP-GP optimal, tight bound is computed using the information on kernel and inducing point locations with Eqs.~\eqref{eq:sensqe} and~\eqref{eq:senshes}. For DP-GP generic, the bound is computed using Eq.~\eqref{eq:sensnd}. The error bars represent one standard error of the mean after 100 repeats.}
    \label{fig:dpgp-dpbbb}
\end{figure}

\subsection{Inference with Real-World Data}

Finally, we test the DP-GP inference on 2-dimensional real-world data that contains the ages, weights, and heights of 5189 Pakistani girls~\citep{asif_dataset_2021}. Our task is to build GP model for predicting height from age and weight while preserving the privacy of both inputs and outputs. We compare our method with a DP version of Bayes by Backprop~\citep{blundell_weight_2015}, DP-BBB~\citep{zhang_differentially_2021} using the Fourier accountant~\citep{koskela_computing_2020} for privacy accounting. In Fig.~\ref{fig:dpgp-dpbbb}, we show the results of applying DP-GP with two different sensitivity computations and DP-BBB to the task. The optimal kernel norm upper bound is computed using the specific information on the kernel function, its hyperparameters, and the locations of the inducing points using Eq.~\eqref{eq:sensqe}. For generic computation, we use Eq.~\eqref{eq:sensnd}.

The models are trained on a randomly sampled 50\% of the data and validated on the rest. For DP-GP, we use a regular 3 by 3 grid of inducing points and the average height in the training set computed privately with \textsc{CoinPress} for the prior mean. The kernel is a product of variance $2.2^2$ and two 1-dimensional exponentiated quadratic kernels with lengthscales of 3 and 3.5 for age and weight dimensions. For DP-BBB, we set the learning rate to $10^{-2}$ and use a two-layered network with 200 hidden units similarly to \citet{zhang_differentially_2021}. We train the model for 200 epochs with 40 batches. We validate the model by sampling the trained network for 100 predictions per data point.

With low values of $\epsilon$, DP-GP optimal provides improvements over DP-BB in terms of both RMSE and predictive log-likelihood and the optimal bound for DP-GP produces better results than the generic bound for all levels of privacy. When the value of $\epsilon$ is high, DP-BBB outperforms DP-GP with regard to both metrics, although it runs for a significantly longer time. The runtime for DP-GP and DP-BBB with one set of parameters is $0.11 \pm 0.0051$ and $3.6 \pm 0.033$ seconds (mean $\pm$ standard deviation), respectively.

\section{Discussion}

We have presented the basic algorithms needed for learning GPs under DP protection for both inputs and outputs, significantly improving privacy over previous privacy-preserving GPs. DP-GP publishes a variational approximation on a set of fixed inducing points. This is computed using analytical Gaussian mechanism for releasing a function of input-vs-inducing-point cross kernels and outputs.

The largest remaining theoretical challenge for GP inference part is the need to add the diagonal regularising term to ensure the perturbed second moment matrix is positive definite. Avoiding this is not easy, but it would be interesting to explore if the methods proposed for linear regression by \citet{sheffet_old_2019} could help.

Because our mechanism operates on the kernel matrix rather than input data, privacy bounds depend on the number of inducing points but not on input dimensionality. In practice, input dimensionality will impact the method as it affects the number of inducing points required for a good approximation. The method also works on non-vectorial data such as graphs or strings.

Important questions for future work include optimising the ratio $\sigma_a/\sigma_b$ that we assumed to be even, possibly by making use of terms of $\bS_2$. Co-developing kernels with small $\|\bK_{ZX_i}\|_2$ such as limited range kernels and better bounds can help improve utility.  Finding more efficient algorithms for hyperparameter learning and finding good inducing inputs, as well as testing the method in different practical applications would also be important.

\subsection*{Acknowledgements}

The authors would like to thank Michalis Titsias and Razane Tajeddine for useful comments and suggestions.
This work has been supported by the Academy of Finland (Finnish Center for Artificial Intelligence FCAI and grant 325573) as well as by the Strategic Research Council at the Academy of Finland (grant 336032).

\bibliographystyle{myabbrvnat}
\bibliography{references}

\appendix

\section{Sensitivity for the Gaussian mechanism for DP-GP inference}

Following \citet{kulkarni_differentially_2021}, we define functions
\begin{align}
    t_a(\bx) &= \bx, \label{eq:t1_appendix} \\
    t_b(\bx) &= [x_1^2 \cdots x_d^2 \;\; \sqrt{2}x_1 x_2 \cdots \sqrt{2}x_{d-1}x_d]^T. \label{eq:t2_appendix}
\end{align}

Using this notation, the quantities $\bA$ and $\bB$ can be expressed as sums of $y_i t_a(\bk_i)$ and $t_b(\bk_i)$, respectively.

The following Lemma and its proof are analogous to Lemma 3.3 of \citet{kulkarni_differentially_2021}:
\begin{lemma}
Let $t_a$ and $t_b$ be defined as in Eqs. \eqref{eq:t1_appendix} and \eqref{eq:t2_appendix} and let $\sigma_a, \sigma_b > 0$. Consider the mechanism
\[ \mathcal{M}(\bX, \bY) = \sum_{i=1}^{|X|} \begin{bmatrix} y_i t_a(\bk_i) \\ t_b(\bk_i) \end{bmatrix}
    + \mathcal{N}\left(0, \mathrm{diag}(\sigma_a^2 \bI_{d}, \sigma_b^2 \bI_{d_2})\right),
\]
where $d_2 = \binom{d+2}{2}$, with $d = |Z|$. Assuming $\|\bk_i\|_2 \le R_k$ and $|y_i| \le R_y$, $\mathcal{M}$ can be reduced to a Gaussian mechanism with noise variance $\sigma_a^2$ and sensitivity
\begin{equation}\label{eq:dpsens_appendix}
    \Delta = \sqrt{\frac{\sigma_b^2}{2 \sigma_a^2}R_y^4 + 2 R_y^2 R_k^2 + 2\frac{\sigma_a^2}{\sigma_b^2} R_k^4}.
\end{equation}
\end{lemma}

\begin{proof}
As $\mathcal{M}$ is a sum over terms each depending on a single data entry, we can compute the sensitivity as the maximal change in a single entry of the sum when the input changes from $(\bx, y)$ to $(\bx', y')$. Denoting $\bk = \bk_{Z\bx}, \bk' = \bk_{Z\bx'}$, we can bound the first term
\begin{equation}
    \| y t_a(\bk) - y' t_a(\bk') \|_2^2 
    = \| y \bk \|^2 + \| y' \bk' \|^2 - 2\langle y\bk, y' \bk' \rangle
    \le 2R_k^2 R_y^2 - 2yy' \langle \bk, \bk' \rangle
    \label{eq:ta_sens}
\end{equation}
and the second term
\begin{equation}
    \| t_b(\bk) - t_b(\bk') \|_2^2 
    = \| \bk \|^4 + \| \bk' \|^4 - 2\langle \bk, \bk' \rangle^2
    \le 2R_k^4 - 2 \langle \bk, \bk' \rangle^2.
    \label{eq:tb_sens}
\end{equation}

Similarly as \citet{kulkarni_differentially_2021}, we can rescale the two parts to share the same variance $\sigma_a^2$ by noticing that
\begin{equation}
\begin{split}
    \mathcal{M}(\bX, \bY) &= 
    \sum_{i=1}^{|X|} \begin{bmatrix} y_i t_a(\bk_i) \\ t_b(\bk_i) \end{bmatrix}
    + \mathcal{N}\left(0, \mathrm{diag}(\sigma_a^2 \bI_{d}, \sigma_b^2 \bI_{d_2})\right) \\
    &= 
    \begin{bmatrix}
    \bI_d & 0 \\ 0 & \frac{\sigma_b}{\sigma_a}
    \end{bmatrix}
    \left( \sum_{i=1}^{|X|} \begin{bmatrix} y_i t_a(\bk_i) \\ \frac{\sigma_a}{\sigma_b} t_b(\bk_i) \end{bmatrix}
    + \mathcal{N}\left(0, \sigma_a^2 \bI_{d + d_2})\right) \right).
\end{split}
\end{equation}

We can bound the sensitivity of $\sum_{i=1}^{|X|} \begin{bmatrix} y_i t_a(\bk_i) \\ \frac{\sigma_a}{\sigma_b} t_b(\bk_i) \end{bmatrix}$ using Eqs. \eqref{eq:ta_sens} and \eqref{eq:tb_sens} as
\begin{equation}
    \Delta^2 \le -2ct^2 - 2yy't + 2cR_k^4 + 2 R_k^2 R_y^2,
\end{equation}
where we have denoted $c = \frac{\sigma_a^2}{\sigma_b^2}, t = \langle \bk, \bk' \rangle$. Viewed as a second order polynomial over $t$, RHS has its maximum at $t = -\frac{yy'}{2c}$ which yields
\begin{equation}
    \Delta^2 \le 2cR_k^4 + 2 R_k^2 R_y^2 + \frac{(yy')^2}{2c}
    \le 2cR_k^4 + 2 R_k^2 R_y^2 + \frac{R_y^4}{2c},
\end{equation}
which can be simplified to yield the claim.
\end{proof}

\section{Approximate bounds for validation set log-likelihood for hyperparameter learning}

Hyperparameter learning uses an estimate of mean validation set log-likelihood evaluated using \textsc{CoinPress} \citep{biswas_coinpress_2020}. \textsc{CoinPress} requries bounds for the function but is not extremely sensitive to the bounds because of the multi-step protocol. In this section, we derive approximate bounds for validation set log-likelihoods. Values outside these bounds will be clipped to make them adhere to the bounds.

The validation set density is a normal with mean given by the GP prediction and variance equal to GP predictive variance plus observation noise variance $\sigma^2$. In order to obtain approximate bounds, we consider values obtained with either perfect mean prediction or one that is off by $2R_y$, where $R_y \ge |y|$ is the bound used in the sensitivity calculations, both under variance equal to just the observation noise. The upper bound for the log-likelihood is tight, but the "lower bound" is not an actual bound as even smaller log-likelihoods are possible if the GP predictions are more than $2 \max |y|$ off. Such values are simply clipped.

These yield clipping bounds for the log-likelihood $L$
\begin{align}
    L_{\text{max}} &= -\frac{1}{2} \log(2 \pi) - \log \sigma \\
    L_{\text{min}} &= -\frac{1}{2} \log(2 \pi) - \log \sigma - \frac{4R_y^2}{2\sigma^2}.
\end{align}

These correspond to \textsc{CoinPress} parameters for the center and radius of the region of interest
\begin{align}
    R &= \frac{R_y^2}{\sigma^2} \\
    C &= -\frac{1}{2} \log(2\pi) - \log \sigma - \frac{R_y^2}{\sigma^2}.
\end{align}

\section{Proof of Theorem 4 (hyperparameter learning privacy)}

\begin{proof}
We show that Algorithm 2 is $(\epsilon_{\text{tot}}, \delta_{\text{tot}})$-DP with respect to the combination of the training and validation sets $((\bX_1, \bY_1), (\bX_2, \bY_2))$.

We start by defining $\delta_2$ to minimise $\delta_{\text{tot}}$. By using the definition of $T$, we find the optimal value
\begin{equation}
    \delta_2 = \frac{\sqrt{2 \delta}}{\gamma}.
\end{equation}
Using this value,
\begin{equation}
    \delta_{\text{tot}} = \frac{\sqrt{2 \delta}}{\gamma} \left[1 - \log\left(\frac{\sqrt{2 \delta}}{\gamma}\right)\right]
    = t [1 - \log t],
\end{equation}
where $t = \frac{\sqrt{2 \delta}}{\gamma}$.
We solve this equation numerically to find $t_0$ that matches the desired $\delta_{\text{tot}}$, allowing us to solve for $\delta = \gamma^2t_0^2/2$ that yields the desired $\delta_{\text{tot}}$.

Given $\delta$, we can easily solve for $\epsilon$ required to match the target $\epsilon_{\text{tot}}$.

Lemma 3 shows that the combination of \textsc{DP-GP-inference} and \textsc{CoinPress-mean} is $(\epsilon, \delta)$-DP with respect to the combination of training and validation sets.

We then use the finite stopping variant of Algorithm 2 of \citet{liu_private_2019}, which is $(\epsilon_{\text{tot}}, \delta_{\text{tot}})$-DP by Theorem 3.4 of \citet{liu_private_2019} when applied to $(\epsilon, \delta)$-DP inference-and-evaluation algorithm.
\end{proof}

\end{document}